\documentclass{article}

\usepackage{arxiv}

\usepackage[utf8]{inputenc} 
\usepackage[T1]{fontenc}    
\usepackage{hyperref}       
\usepackage{url}            
\usepackage{booktabs}       
\usepackage{amsfonts}       
\usepackage{amsmath}
\usepackage{amssymb}
\usepackage{amsthm}
\usepackage{stmaryrd}
\usepackage{mathtools}
\usepackage{nicefrac}       
\usepackage{microtype}      
\usepackage{lipsum}		
\usepackage{graphicx}
\usepackage{natbib}
\usepackage{doi}

\usepackage{svg}
\usepackage{booktabs}
\usepackage{bm}

\newcommand{\uniteandlead}{\textsc{UnL}}
\newcommand{\textcumulative}{\texttt{Cumulative}}
\newcommand{\cumulative}[4]{\ensuremath{\mathtt{Cumulative}(#1, #2, #3, #4)}}
\newcommand{\capacity}[3]{\ensuremath{\mathcal{L}(#1, #2, #3)}}
\newcommand{\disjunctive}{\texttt{Disjunctive}}

\newcommand{\occ}[2]{\ensuremath{\langle {#1} \mid {#2} \rangle}}

\newcommand{\poly}[2]{\ensuremath{\mathcal{P}(#1,#2)}}

\DeclareMathOperator*{\argmin}{arg\,min}
\DeclareMathOperator*{\argmax}{arg\,max}

\AtBeginEnvironment{definition}{\DeclareEmphSequence{\bfseries\itshape}}
\AtBeginEnvironment{lemma}{\DeclareEmphSequence{\bfseries\itshape}}
\AtBeginEnvironment{theorem}{\DeclareEmphSequence{\bfseries\itshape}}
\AtBeginEnvironment{proposition}{\DeclareEmphSequence{\bfseries\itshape}}

\newcommand{\pumpkin}{\ifx\authoranonymous\relax \textcolor{red}{solver}\else \text{Pumpkin}\fi\xspace}
\newcommand{\Pumpkin}{\ifx\authoranonymous\relax \textcolor{red}{Solver}\else \text{Pumpkin}\fi\xspace}
\newcommand{\PumpkinUrl}{\ifx\authoranonymous\relax \textcolor{red}{Anonymous URL}\else \url{https://github.com/ConSol-Lab/Pumpkin}\fi}
\newcommand{\delftblue}{\ifx\authoranonymous\relax \textcolor{red}{an HPC cluster}\else \text{DelftBlue}~\citep{DelftHighPerformanceComputingCentre2024-Other}\fi\xspace}

\newcommand{\nlowerbounds}{25}
\newcommand{\nupperbounds}{five}
\newcommand{\nsearchless}{eight}

\usepackage{float}
\usepackage{subcaption}

\usepackage{xcolor}
\definecolor{Comment}{RGB}{1,80,50}
\usepackage[noend]{algorithm2e}

\SetCommentSty{mycommfont}
\SetKw{Continue}{continue}
\SetKw{Break}{break}
\SetKw{True}{true}

\usepackage{thmtools}
\usepackage{cleveref}

\Crefname{algocf}{Algorithm}{Algorithms}
\AddToHook{cmd/appendix/before}{\crefalias{section}{appendix}}
\newtheorem{theorem}{Theorem}
\newtheorem{lemma}[theorem]{Lemma}
\newtheorem{proposition}[theorem]{Proposition}
\newtheorem{definition}[theorem]{Definition}
\theoremstyle{definition}
\newtheorem{example}[theorem]{Example}
\theoremstyle{remark}
\newtheorem{note}[theorem]{Note}

\title{On inferring cumulative constraints}

\author{
\href{https://orcid.org/0009-0009-0655-4200}{\includegraphics[scale=0.06]{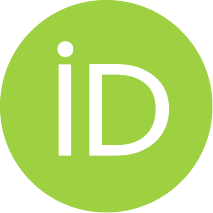}\hspace{1mm}Konstantin Sidorov} \\
	Faculty of Electrical Engineering,\\ Mathematics \& Computer Science \\
	Delft University of Technology\\
	Delft, The Netherlands \\
	\texttt{k.sidorov@tudelft.nl}
}

\renewcommand{\shorttitle}{On inferring cumulative constraints}

\hypersetup{
pdftitle={On inferring cumulative constraints},
pdfsubject={cs.AI,math.OC},
pdfauthor={Konstantin Sidorov},
pdfkeywords={Constraint programming, scheduling, cumulative, lifting, valid inequalities},
}

\begin{document}
\maketitle

\begin{abstract}
Cumulative constraints are central in scheduling with constraint programming, yet propagation is typically performed per constraint, missing multi-resource interactions and causing severe slowdowns on some benchmarks. I present a preprocessing method for inferring additional cumulative constraints that capture such interactions without search-time probing.
This approach interprets cumulative constraints as linear inequalities over occupancy vectors and generates valid inequalities by (i) discovering covers, the sets of tasks that cannot run in parallel, (ii) strengthening the cover inequalities for the discovered sets with lifting, and (iii) injecting the resulting constraints back into the scheduling problem instance. Experiments on standard RCPSP and RCPSP/max test suites show that these inferred constraints improve search performance and tighten objective bounds on favorable instances, while incurring little degradation on unfavorable ones. Additionally, these experiments discover \nlowerbounds{} new lower bounds and \nupperbounds{} new best solutions; eight of the lower bounds are obtained directly from the inferred constraints.
\end{abstract}


\section{Introduction}
\label{section:intro}

Cumulative constraints and their use in scheduling problems are one of the longstanding success stories of constraint programming (CP), serving as a central abstraction and a source of powerful inferences about resource-constrained scheduling problems. A large body of work has demonstrated how powerful propagation on a single cumulative resource can be, including effective detection of overload situations, pruning based on resource envelopes, and increasingly sophisticated explanations and propagator combinations. All these developments have contributed to the state-of-the-art performance of CP solvers on industrial scheduling benchmarks.

Nevertheless, most existing techniques in this space reason about one cumulative constraint at a time. Even when a model contains several resources, propagation is typically performed independently per resource, with indirect interactions only mediated through domain narrowing of start-time variables. However, the recent work by~\citet{Sidorov2025-Other} has made a first step to expose the limitations of this “single-resource viewpoint,” including the \emph{\( 3n \) problem} structure, a pathological case when an instance with three cumulative constraints is equivalent to a trivially infeasible single-machine scheduling problem but is intractable for a conventional lazy clause-generating solver, a paradigm subsuming many of the state-of-the-art CP codes.

While Sidorov et al. have addressed a special (disjunctive) case of this problem with a \emph{Unite and Lead} (further referred to as \uniteandlead{}) approach and managed to discover large hidden structures in several known research benchmarks, their technique has two fundamental limitations.
First, while it succeeds at discovering disjunctive structures (that is, the collections of tasks that cannot be run in parallel), it does not support any extensibility for deriving more general multi-resource interactions. Second, \uniteandlead{} has to explicitly probe for disjunctive cliques throughout the search, incurring significant overhead even in instances where no such structure exists. Since the clique search overhead is not offset by stronger reasoning on instances without a disjunctive structure, the solver performance deteriorates sharply on such problems. As a result, Sidorov et al. show that, when \uniteandlead{} does not improve the search, it commonly slows down the solver by more than an order of magnitude.

In this paper, I address these limitations by proposing a new algorithmic framework for inferring cumulative constraints at the root search node. Instead of restricting attention to disjunctive reasoning, this approach operates over cumulative constraints as \emph{linear inequalities} on a special representation of start time variables (occupancy vectors), and reasoning across multiple resources reduces to \emph{generating valid cutting planes} in the same sense as it is understood in integer programming. This viewpoint subsumes a wide range of multi-constraint reasoning techniques from earlier work, including the version of \uniteandlead{} restricted to the root node and classical lower-bound techniques for RCPSPs.

Next, I leverage this connection by describing a workflow for inferring cumulative constraints based on the idea of \emph{lifting}~\citep{Padberg1975-operRes}. Lifting is a general-purpose technique from mixed-integer programming for tightening the generated cutting planes \( \boldsymbol{a}^T \boldsymbol{x} \le b \) by picking a variable \( y \) not mentioned in the inequality and finding the largest \( \beta \) such that \( \boldsymbol{a}^T \boldsymbol{x} + \beta y \le b \) is still valid. I apply this idea by (i) discovering sets of tasks \( C \) that cannot be run in parallel, (ii) lifting the constraint of the form “tasks in \( C \) cannot consume more than \( (\#C - 1) \) units of resource,” and (iii) introducing the lifted constraints as new cumulative constraints into the model.

I evaluate the method as a preprocessing step for two state-of-the-art CP solvers on a suite of standard RCPSP benchmarks. The results suggest that the approach described in this paper successfully retains most of the positive traits of \uniteandlead{} (substantial gains on instances with hidden structure), but does so without inheriting its instability. This is helpful for advancing the best-known bounds for the benchmark instances: with the presented approach, I report the improved schedules for \nupperbounds{} benchmarks and tighten the lower bounds for \nlowerbounds{} instances, with \nsearchless{} of them not requiring any search to verify the correctness of the new bound.

The remainder of the paper is organized as follows. \Cref{section:preliminaries} introduces the notions of constraint programming and polyhedral geometry that are necessary for the rest of the presentation. \Cref{section:occupancy-ineq} establishes the connection between CP models with cumulative constraints and the notion of valid inequality. \Cref{section:approach} presents the lifting-based approach to inferring cumulative constraints. I evaluate this approach in \Cref{section:experiments} from two angles: (a) as a preprocessing technique for adding cumulative constraints to the input instance before running a CP solver, and (b) as a technique for inferring cumulative constraints with unit capacity,\footnote{Also known as disjunctive or no-overlap constraints.} by comparing it with the \uniteandlead{} approach. \Cref{section:related-work} discusses earlier methods for joint reasoning on cumulative constraints and their connection with the lifting approach. Finally, \Cref{section:conclusions} concludes the narrative and discusses future research directions.

\section{Preliminaries}
\label{section:preliminaries}
\subsection{Constraint programming}

A constraint satisfaction problem (CSP) consists of a tuple \((\mathcal{X}, \mathcal{C}, \mathcal{D})\) where \(\mathcal{X}\) is the set of \emph{variables}, \(\mathcal{C}\) is the set of \emph{constraints} which specify the relations between variables, and \(\mathcal{D}\) is the \emph{domain} which specifies for each variable which values it can take~\citep{Rossi2006-Other}. A \emph{solution} \(\mathcal{I}\) is a mapping that maps each variable in \(\mathcal{X}\) to a single value in the domain of that variable in \(\mathcal{D}\) which satisfies all of the constraints in \(\mathcal{C}\).

Constraint programming (CP) is a paradigm for solving CSPs; CP solvers enforce constraints through \emph{propagators}, each represented with a subroutine that removes values from \(\mathcal{D}\) infeasible under the constraints in \(\mathcal{C}\). After applying the propagators, the solver makes a decision that creates several subproblems by splitting the domain of a variable into two or more parts. This process of applying propagators and making decisions is performed until either a solution \(\mathcal{I}\) is found, the problem is found to be unsatisfiable, or a termination criterion is met.

\textcumulative{} is a constraint useful for many scheduling problems to model limited renewable resources, such as available work-hours. In \Cref{def:cumulative}, the variable \( s_i \) encodes the start time of task $i$, and \( (s_i + d_i) \) encodes its finish time. Thus, we implicitly associate each task \( i \in T \) with an interval \( [s_i, s_i+d_i) \).
\begin{definition}
    \label{def:cumulative}
    Let \( T \) be a set of \emph{tasks}, and let a task $i \in T$ be defined by its start variable $s_i$, resource usage $r_i \in \mathbb{Z}_{\ge 0}$, and duration $d_i \in \mathbb{Z}_{\ge 0}$. Finally, let \( C \in \mathbb{Z}_{\ge 0} \) be the \emph{capacity} of the resource. Then the \( \cumulative{\boldsymbol{s}}{\boldsymbol{r}}{\boldsymbol{d}}{C} \) constraint is the condition that at any time point \( \tau \) the cumulative resource usage of intervals \( [s_i, s_i+d_i) \) covering \( \tau \) does not exceed the capacity: 
    \begin{equation}
        \label{eq:cumulative}
        \forall \tau \in \mathbb{Z}\ :\ \sum_{i \in T\ :\ s_i \le \tau < s_i + d_i} r_i \leq C.
    \end{equation}
\end{definition}

An important special case of \textcumulative{} is where the resource has unit capacity \( C = 1 \) and each task has resource usage \( r_i \in \{ 0, 1\} \), known as \disjunctive{}. This special case is significant because many inference procedures that are intractable for \textcumulative{} constraints can be efficiently executed for \disjunctive{} constraints. For example, determining satisfiability of a single \textcumulative{} constraint is NP-complete~\citep{Baptiste1999-annOperRes}, but determining satisfiability of a single \disjunctive{} constraint can be done in polynomial time~\citep[Section 3]{Vilim2004-Other-01}.

\subsection{Polyhedral geometry}

In this work, we consider sets of the form \( \poly{A}{\boldsymbol{b}} := \{ \boldsymbol{x} \in \{ 0, 1 \}^n :  A\boldsymbol{x} \le \boldsymbol{b} \} \); unless stated otherwise, \textbf{we assume that all coefficients are nonnegative integers}, that is, \( A \in \mathbb{Z}_{\ge 0}^{m \times n} \) and \( \boldsymbol{b} \in \mathbb{Z}_{\ge 0}^{m} \). As known from the integer programming theory~\citep{Ziegler1995-Other}, those sets can also be seen as vertices of some polyhedron \( \{ \boldsymbol{x} : A^* \boldsymbol{x} \le \boldsymbol{b}^*, \boldsymbol{0} \le \boldsymbol{x} \le \boldsymbol{1} \} \), albeit with a different set of constraints given by \( A^* \) and \( \boldsymbol{b}^* \). This motivates the following definition:

\begin{definition}
    An inequality \( \boldsymbol{\pi}^T \boldsymbol{x} \le \pi_0 \), also written as \( (\boldsymbol{\pi}, \pi_0) \in \mathbb{R}_{\ge 0}^{n+1} \), is called a \emph{valid inequality} for \( \poly{A}{\boldsymbol{b}} \) if it holds for any point in \( \poly{A}{\boldsymbol{b}} \).
\end{definition}

Discovering valid inequalities separating a non-integer vertex \( \boldsymbol{x}^* \) of the continuous relaxation of \( \poly{A}{\boldsymbol{b}} \)  from \( \poly{A}{\boldsymbol{b}} \) itself is a theoretically rich topic with major practical implications for integer programming solvers. For the purposes of this work, we only need a few core notions; an interested reader is encouraged to survey the literature on knapsack polyhedra~\citep{Hojny2020-annOperRes} for a more systematic exposition. Consider first a simple family of valid inequalities:

\begin{definition}
    Given a polyhedron \( \poly{\boldsymbol{a}^T}{b} \) with a single inequality constraint \( \boldsymbol{a}^T \boldsymbol{x} \le b \), we say that \( C \subseteq \{ 1, \dotsc, n \} \) is its \emph{cover} if \( \boldsymbol{a}(C) := \sum_{i \in C}a_i > b \).
\end{definition}

\begin{proposition}
    A \emph{cover inequality} \( \boldsymbol{x}(C) \le \# C - 1\) is a valid inequality for \( \poly{\boldsymbol{a}^T}{b} \) if \( C \) is its cover.
\end{proposition}

Next, we need a general-purpose procedure for increasing the left-hand side of a valid inequality \( \boldsymbol{\pi}^T \boldsymbol{x} \le \pi_0 \)—which may or may not be a cover inequality—known as \emph{lifting}~\citep{Padberg1975-operRes}. This procedure maintains the set \( C \) of variables that have already been considered for adding to the inequality, starting with \( C \gets \{ i : \pi_i \neq 0 \} \), and repeating the following steps until \( C = \{ 1, \dotsc, n \} \):
\begin{enumerate}
    \item Choose an unused variable \( i \in \{ 1, \dotsc, n \} \setminus C \) or terminate if none exist.
    \item Solve an auxiliary subproblem over the variables in \( C \) \begin{align}
        \sum_{c \in C} \pi_c x_c &\to \max \nonumber \\
        \sum_{c \in C} a_{j,c} x_c &\le b_j - a_{j,i} \forall j \label{eq:lifting-subproblem} \\
        x_c &\in \{ 0, 1 \} \, \forall c \in C \nonumber
    \end{align}
    that corresponds to maximizing the left-hand side of the current inequality \( \boldsymbol{\pi}^T \boldsymbol{x} \) when the variable \( x_i \) under consideration is set to 1.
    \item Set \( C \gets C \cup \{ i \} \) and \( \pi_i \gets \pi_0 - v^*(\boldsymbol{\pi}, \pi_0, C; A, \boldsymbol{b}) \), where \( v^*(\boldsymbol{\pi}, \pi_0, C, A, \boldsymbol{b}) \) is the maximum value of the objective of the problem given by \Cref{eq:lifting-subproblem}.
\end{enumerate}

To illustrate the concepts of cover inequality and lifting, we apply all of them in the following example, which is also used later in the paper to illustrate the proposed approach:

\begin{example}
    \label{example:lifting-prelims}
    Consider the following knapsack problem:
    \begin{align}
        x_1 + x_2 + x_3 + x_4 &\to \max \nonumber \\
        5x_1 + 3x_2 + 2x_3 + 4x_4 &\le 7 \label{example:lifting-prelims:knapsack-inequality} \\
        \boldsymbol{x} &\in \{ 0, 1 \}^4.  \nonumber
    \end{align}

    The upper bound on the objective derived by relaxing the integrality constraints\footnote{Also known as the linear programming relaxation.} is equal to \( 2\frac{1}{2} \) and is achieved on the solution \( \boldsymbol{x}^{(1)} = \left(0, 1, 1, \frac{1}{2}\right)\). The support set \( C = \{ 2, 3, 4 \} \) of this solution is the \emph{cover} for \Cref{example:lifting-prelims:knapsack-inequality} (\( 3 + 2 + 4 = 9 > 7\)), and it corresponds to the \emph{cover inequality} \( x_2 + x_3 + x_4 \le 2 \) forbidding the simultaneous choice of those three variables.

    Since this is a valid inequality, it is possible to add it to the optimization model alongside \Cref{example:lifting-prelims:knapsack-inequality} and solve the resulting linear programming relaxation; this yields an upper bound of \( 2\frac{2}{5} \) and is achieved on the solution \( \boldsymbol{x}^{(2)} = \left( \frac{2}{5}, 1, 1, 0 \right) \). We can improve this bound by applying the \emph{lifting} procedure first; in this example, we lift the only remaining variable \( x_1 \)  with the following reasoning.
    
    Consider a constraint \( \alpha x_1 + x_2 + x_3 + x_4 \le 2 \) for arbitrary \( \alpha \ge 0 \) and observe that it holds for any solution satisfying \Cref{example:lifting-prelims:knapsack-inequality} with \( x_1 = 0 \) regardless of \( \alpha \). Suppose that \( x_1 = 1 \), and rewrite the lifted constraint and \Cref{example:lifting-prelims:knapsack-inequality} as follows: \begin{equation}
        \underbrace{3x_2 + 2x_3 + 4x_4 \le 2}_{\text{\Cref{example:lifting-prelims:knapsack-inequality}}}  \implies \underbrace{\alpha \le 2 - ( x_2 + x_3 + x_4)}_{\text{lifted constraint for }x_1 = 1}. \label{example:lifting-prelims:lifting-condition}
    \end{equation}
    Since we look for the tightest valid constraint, we can achieve this by choosing the largest \( \alpha \) such that \Cref{example:lifting-prelims:lifting-condition} is still true for any \( x_2, x_3, x_4 \in \{0, 1 \}\). Specifically, we achieve this by setting \begin{equation}
        \label{example:lifting-prelims:lifting-subproblem}
        \alpha \gets 2 - \max \{ x_2 + x_3 + x_4 \vert 3x_2 + 2x_3 + 4x_4 \le 2, x_2, x_3, x_4 \in \{ 0, 1 \} \},
    \end{equation}  which also happens to be the problem written in \Cref{eq:lifting-subproblem} for the values of \( (A, \boldsymbol{b}, C, \boldsymbol{\pi}, \pi_0) \) used in this example. To solve the \emph{lifting subproblem} in \Cref{example:lifting-prelims:lifting-subproblem}, observe that \( x_3 = 1, x_2 = x_4 = 0 \) is a feasible solution, but setting any two variables to one is not possible under the \( 3x_2 + 2x_3 + 4x_4 \le 2 \) constraint. Therefore, setting \( \alpha \gets 2 - 1 = 1\) retains the validity of the inequality, and we obtain the lifted inequality \( \boxed{x_1 + x_2 + x_3 + x_4 \le 2.} \) To conclude the example, observe that adding it alongside \Cref{example:lifting-prelims:knapsack-inequality} and solving the linear programming relaxation yields an upper bound of \( 2 \), achieved on, among others,  \( \boldsymbol{x}^* = (0, 1, 0, 1) \), which is enough to conclude optimality.
\end{example}

\section{\texorpdfstring{From \textcumulative{} to linear inequalities}{From Cumulative to linear inequalities}}
\label{section:occupancy-ineq}

We start by introducing a representation of the start-time variables that we need in order to reason about linear inequalities:

\begin{definition}
    \label{def:occupancy-vector}
    Given a variable \( x \) of a CSP and an integer \( d \), further referred to as \emph{time} and \emph{duration}, we associate the variable assignments with \emph{occupancy vectors} \( \occ{x}{d} \in \{ 0, 1 \} ^\mathbb{Z} \) as follows: \[ \occ{x}{d}(t) := \begin{cases}
        1 & \text{if } x \le t < x + d \\
        0 & \text{otherwise}. \\
    \end{cases} \]
\end{definition}

The key point behind this definition is that a \textcumulative{} constraint can be rewritten as a \emph{linear inequality}\footnote{For simplicity, an inequality between a scalar and a vector \( \boldsymbol{x} \le b \) is understood as the coordinate-wise inequality \( \boldsymbol{x} \le (b, \dotsc, b)\).} over the occupancy vectors: \begin{equation}
    \cumulative{\boldsymbol{x}}{\boldsymbol{r}}{\boldsymbol{d}}{C} \Longleftrightarrow \sum_{i \in T} r_i \occ{x_i}{d_i} \le C. \label{eq:cumulative-to-linear}
\end{equation} 

This reformulation is similar to the time-resource decomposition~\citep{Schutt2009-Other} in the sense that it linearizes a \textcumulative{} constraint, but the main difference is that \Cref{eq:cumulative-to-linear} introduces \emph{one vector} inequality, which is possible through the use of occupancy vectors. Conversely, the time-difference decomposition introduces one constraint per time point, as well as a collection of constraints that enforce \Cref{def:occupancy-vector}. While using \Cref{eq:cumulative-to-linear} directly does not provide any advantage for solving CSPs with \textcumulative{} constraints, it provides a new viewpoint for such models: they can be seen as models with \emph{polyhedral} constraints over occupancy vectors, with the coefficients of inequalities mapping one-to-one to the arguments of \textcumulative{} constraints.

Given that, it is natural to assume that if a part of the CP model can be restated in a polyhedral form, then the cutting-plane approaches could be used to derive new linear inequalities—or, in CP terms, new \textcumulative{} constraints—with potentially better inference through aggregation. The next result formalizes this insight and sets the stage for a more practical discussion of inferring \textcumulative{} constraints:

\begin{lemma}
    \label{lemma:cumulative-cuts}
    Consider a CSP \( (\mathcal{X}, \mathcal{D}, \mathcal{C}) \) with \( \mathcal{X} = (x_1, \dotsc, x_n) \) such that all satisfying assignments also satisfy the constraints \( \sum_{j=1}^n a_{ij} \occ{x_j}{d_j} \le b_i, 1 \le i \le m \) for some \( A \in \mathbb{Z}_{\ge 0}^{m \times n} \) and \( \boldsymbol{b} \in \mathbb{Z}_{\ge 0}^m \). Let \( (\boldsymbol{\pi}, \pi_0) \) be a valid inequality for the polyhedron \( \mathcal{P}(A, \boldsymbol{b}) \). Then any assignment satisfying the CSP also satisfies \( \cumulative{\mathcal{X}}{\boldsymbol{\pi}}{\boldsymbol{d}}{\pi_0} \).
\end{lemma}

\begin{note}
    We assume in this lemma that all linear inequalities on occupancy vectors use the same duration variables for any given time variable. This is a natural assumption, since \( d_i \) commonly refers to the duration of \( i \)-th task, which does not depend on the considered resource.
\end{note}

\begin{proof}
    Suppose that this is not the case; then, by \Cref{eq:cumulative-to-linear}, there is a satisfying assignment \( \theta \) of the CSP such that \( \sum_{j=1}^n a_{ij} \occ{x_j}{d_j} \le b_i, 1 \le i \le m \) are true but \( \sum_{j=1}^n \pi_j \occ{x_j}{d_j} \le \pi_0 \) is not. Let \( t \) be the time point where \( \sum_{j=1}^n \pi_j \occ{x_j}{d_j} \le \pi_0 \) does not hold, that is, \( \sum_{j=1}^n \pi_j y_j > \pi_0 \) for \( y_j = \occ{x_j}{d_j}(t) \).

    We thus have constructed a vector \( \boldsymbol{y} \in \{ 0, 1 \}^n \) such that \( A\boldsymbol{y} \le \boldsymbol{b} \) and \( \boldsymbol{\pi}^T \boldsymbol{y} > \pi_0 \). However, that contradicts the definition of the valid inequality, since \( \boldsymbol{y} \) is a point in the polyhedron \( \poly{A}{\boldsymbol{b}} \) for which the valid inequality \( \boldsymbol{\pi}^T \boldsymbol{y} \le \pi_0 \) does not hold.
\end{proof}

To see the utility of \Cref{lemma:cumulative-cuts}, consider the following example:

\begin{example}
    Consider a CSP on variables \( \boldsymbol{x} = (x_1, x_2, x_3, x_4) \) with a constraint \[ \cumulative{\boldsymbol{x}}{(5, 3, 2, 4)}{\boldsymbol{d}}{7} \Longleftrightarrow 5\occ{x_1}{d_1} + 3\occ{x_2}{d_2} + 2\occ{x_3}{d_3} + 4 \occ{x_4}{d_4} \le 7 \] for some vector of durations \( \boldsymbol{d} \). We know from \Cref{example:lifting-prelims} that \( (\boldsymbol{\pi}; \pi_0) = (1, 1, 1, 1; 2) \) is a valid inequality for the \( \poly{A}{\boldsymbol{b}} \) polyhedron formed from the \textcumulative{} constraint. But, by \Cref{lemma:cumulative-cuts}, that means that we can introduce the following constraint to the CSP: \[ \cumulative{\boldsymbol{x}}{(1, 1, 1, 1)}{\boldsymbol{d}}{2} \Longleftrightarrow \occ{x_1}{d_1} + \occ{x_2}{d_2} + \occ{x_3}{d_3} + \occ{x_4}{d_4} \le 2. \]
    In other words, we have discovered that \emph{at most two tasks can run in parallel} in any feasible assignment.\end{example}

More generally, \Cref{lemma:cumulative-cuts} suggests that inferring \textcumulative{} constraints can be done by writing out the model constraints, where possible, in the occupancy-vector form, discovering a valid inequality for those constraints, and substituting its coefficients as usages and capacity for a new “resource.” The discovery step, however, has a subtle difference from the integer programming case: the goal there is usually not simply to produce new inequalities, but to \emph{separate} a non-integer solution \( \boldsymbol{x}^* \) from the linear programming relaxation by a valid inequality~\citep{Kaparis2010-mathProgram}; this, in turn, suggests starting with a cover violated by \( \boldsymbol{x}^* \). In the absence of such guidance, some other strategy for enumerating covers is required. The next section describes these steps—cover generation, lifting, and constraint introduction—in more detail.

\section{\texorpdfstring{Inferring \textcumulative{} constraints by lifting}{Inferring Cumulative constraints by lifting}}
\label{section:approach}

This section introduces an approach for inferring auxiliary \textcumulative{} constraints from \( m \) constraints \( \sum_{j=1}^n a_{ij} \occ{x_j}{d_{j}} \le b_i \) for \( 1 \le i \le m \). To employ the lifting in the absence of a guiding “non-solution,” I propose instead to:

\begin{enumerate}
    \item Find a collection \( \mathcal{C} \) of at most \( N_\mathrm{cover} \) covers for the constraints \( A\boldsymbol{x} \le \boldsymbol{b}, \boldsymbol{x} \in \{0, 1 \}^n \).
    \item For each cover \( C \in \mathcal{C} \), lift the cover inequality \( \boldsymbol{x}(C) \le \# C - 1 \) with respect to the polyhedron \( \poly{A}{\boldsymbol{b}} \) as described in \Cref{section:preliminaries}.
    \item Among the lifted \textcumulative{} constraints, add a limited number \( N_\mathrm{out} \) of them, giving preference to the most restrictive constraints.
\end{enumerate}

Since this procedure is bound to produce many valid inequalities, the next definition introduces a quality metric of a \textcumulative{} constraint. I use it directly in Step 3 by giving preferences to the constraints with the larger score, as well as to choose which of the (many) covers in \( \mathcal{C} \) will eventually be lifted.

\begin{definition}
    Given a constraint \( \cumulative{\boldsymbol{x}}{\boldsymbol{r}}{\boldsymbol{d}}{C}\), we call its \emph{capacity bound} the ratio of total resource usage over the time horizon to its capacity per time unit: \( \capacity{\boldsymbol{r}}{\boldsymbol{d}}{C} := \frac{\sum_i d_i r_i}{C}. \)
\end{definition}

The relevance of this quantity, as pointed out in early RCPSP work~\citep[p.~255]{Stinson1978-aiieTrans}~\citep[Section~2]{Klein1999-eurJOperRes}, is that it gives a lower bound on the duration \emph{spanned} by the tasks as formalized by the following statement, the proof of which is available in the aforementioned RCPSP work~\citep{Stinson1978-aiieTrans}.

\begin{lemma}
    For any satisfying assignment \( \theta \) of a constraint \( \cumulative{\boldsymbol{x}}{\boldsymbol{r}}{\boldsymbol{d}}{C}\), let \( s := \min_i x_i \) and \( f := \max_i (x_i + d_i )\).\footnote{In the scheduling interpretation, \( s \) is the earliest starting time among the constrained tasks, \( f \) is the latest finishing time among these tasks, and their difference \( (f - s) \) is the \textit{makespan} of the plan.} Then \( f - s \ge \capacity{\boldsymbol{r}}{\boldsymbol{d}}{C}\).
\end{lemma}

Step~1 enumerates seed covers independently for each of the \( m \) \textcumulative{} constraints. First, this procedure collects the ``short'' covers of cardinality two and three and keeps the best \( N_{\mathrm{cover}} \) covers according to the capacity bound. Given one such constraint \( \sum_i a_i \occ{x_i}{d_i} \le b \), the ``short'' covers are constructed as follows:
\begin{itemize}
    \item Collect all pairs \( \{ i, j \} \) such that \( a_i + a_j > b\), which is done with time complexity \( O(n^2) \).
    \item For any \emph{non-covering} pair \( \{ i, j \} \), consider the set \( K(i, j) = \{ k : a_k > b - (a_i + a_j) \} \). If it is empty, do nothing; otherwise, add a cover \( \{ i, j, k(i, j) \} \) with \( k(i, j) = \argmax_{k \in K(i,j)} d_k \), that is, complete the cover by the longest eligible task. This step is done with time complexity \( O(bn^2) \).
\end{itemize} Then, this set is extended by adding ``long'' covers (with no a priori limit on the resulting cardinality). Given a constraint \( \sum_i a_i \occ{x_i}{d_i} \le b \), long covers are generated by grouping tasks by their resource consumption values \( v \). For each group \( B_v := \{ i : a_i = v \} \), pick the smallest \( k \) such that \( kv > b \), and add the two uniform seed covers consisting of (i) the \( k \) longest tasks in \( B_v \) and (ii) the \( k \) shortest tasks in \( B_v \); this is done in \( O(bn\log n) \).

The complete procedure can be seen as a consecutive application of \Cref{alg:cover-enumeration} on the original CSP and \Cref{alg:cumulative-generation} on the CSP and the covers constructed in the earlier phase; in both algorithms, \( \mathtt{TopK}(\mathcal{C}, k) \) is a procedure that sorts \( \mathcal{C} \) by descending value of \( \capacity{\cdot}{\boldsymbol{d}}{\cdot} \) and returns the first \( k \) elements. Aside from the steps discussed in this section, the implementation employs two additional optimizations: First, it discards any valid inequalities \( (\boldsymbol{\pi}, \pi_0) \) that are dominated by the model constraints in the sense that \( \boldsymbol{\pi} \le \boldsymbol{a}_i\) and \( \pi_0 \ge b\). Second, after lifting a cover \( C \) to a constraint of the form \( \sum_{k \in C'} x_k + \cdots \le \# C-1 \), all subsets of \( C' \) with cardinality \( \#C \) are marked as covers \emph{which are not lifted later}; to see why this is helpful, consider the following example.

\begin{example}
    Suppose that \Cref{alg:cumulative-generation} has lifted a cover \( C_0 = \{ 1, 2 \} \) to a constraint \begin{equation}
        \underbrace{\occ{x_1}{d_1} + \occ{x_2}{d_2}}_{\text{cover inequality}} + \underbrace{\occ{x_3}{d_3} + \cdots + \occ{x_{100}}{d_{100}}}_{\text{lifted terms}} \le 1. \label{equation:lifting-short-circuit}
    \end{equation}
    In other words, lifting has discovered that no two tasks among the set \( C' = \{ 1, 2, \dotsc, 100 \} \) can run in parallel, which translates to an update \( Q \gets Q \cup (C', 2) \).

    Now, consider later iterations of the loop over covers \( C \) that visit other binary covers within \( C' \); since any two tasks form a cover, this happens for \( \binom{100}{2} - 1 = 4949 \) iterations unless some task pairs were discarded before the lifting. Observe that the “skipping” condition holds for \emph{any two tasks} in \( C' \) since \( (C', 2) \in Q, C' \supseteq C \), and \( \# C \ge 2 \); that means that none of those iterations invokes the \( v^* \) function, and skipping \( \approx 5000 \) redundant calls to \( v^* \) is a performance gain, because (a) solving lifting subproblems is the bottleneck of \Cref{alg:cumulative-generation}, and (b) lifting on \( C' \) is likely to re-discover \Cref{equation:lifting-short-circuit} instead of a new constraint; for example, it is guaranteed to do so if it lifts the variables in \( C' \) before others.
\end{example}

\begin{algorithm}[ht]
\caption{Enumeration of covers.}\label{alg:cover-enumeration}
\KwData{A CSP \( (\mathcal{X}, \mathcal{D}, \mathcal{C}) \).}
\KwData{A polyhedron \( \poly{A}{\boldsymbol{b}} \) and durations \( \boldsymbol{d} \) satisfying the assumptions of \Cref{lemma:cumulative-cuts}.}
\KwData{The maximum number of considered short covers \( N_{\mathrm{cover}}. \)}
\KwResult{A set of valid covers.}
\( \mathcal{C} \gets \emptyset \)\tcc*{The set of all discovered covers}
\tcp{Generate ``short'' covers of cardinalities two and three}
\For{\( r \gets 1 \) \KwTo \( m \)}{
    \For{\( i \gets 1 \) \KwTo \( n \), \( j \gets i+1 \) \KwTo \( n \)}{
        \eIf{\( a_{r,i} + a_{r,j} > b_r\)}{
            \( \mathcal{C} \gets \mathcal{C} \cup \{ \{ i, j \} \} \) \tcc*{add all binary covers}
        }{
            \If{\( K(i, j) \neq \emptyset \)}{
                \( \mathcal{C} \gets \mathcal{C} \cup \{ \{ i, j, k(i, j) \} \} \) \tcc*{complete the ternary cover}
            }
        }
    }
}
\tcp{Retain the \( N_\mathrm{cover} \) covers with the highest capacity bound}
\( \mathcal{C} \gets \mathtt{TopN}(\mathcal{C}, N_\mathrm{cover}) \)\;

\tcp{Generate ``long'' covers of unbounded cardinality}
\For{\( r \gets 1 \) \KwTo \( m \)}{
    \tcp{Group tasks by consumption value \( v \)}
    \ForEach{\( v \) with \(  \{ i : a_{r,i} = v \} \neq \emptyset \)}{
        \tcp{Pick the smallest \( k \) such that \( k \cdot v > b_r \)}
        \( k \gets \min\{ t \in \mathbb{Z} : t \cdot v > b_r \} \)\;
        \( C^{\max} \gets \) the \( k \) indices in \( B_v \) with the largest durations \( d_i \)\;
        \( C^{\min} \gets \) the \( k \) indices in \( B_v \) with the smallest durations \( d_i \)\;
        \tcp{Add the longest and shortest \( k \) tasks as covers}
        \( \mathcal{C} \gets \mathcal{C} \cup \{ C^{\max}, C^{\min} \} \)\;
    }
}
\Return \( \mathcal{C} \)\;
\end{algorithm}

\begin{algorithm}[ht]
\caption{Lifting-based procedure for inferring auxiliary \textcumulative{} constraints.}\label{alg:cumulative-generation}
\KwData{A CSP \( (\mathcal{X}, \mathcal{D}, \mathcal{C}) \).}
\KwData{A polyhedron \( \poly{A}{\boldsymbol{b}} \) and durations \( \boldsymbol{d} \) satisfying the assumptions of \Cref{lemma:cumulative-cuts}.}
\KwData{A set of covers \( \mathcal{C} \).}
\KwData{The maximum number of added constraints \( N_{\mathrm{out}}. \)}
\KwResult{A set of at most \( N_\mathrm{out}\) valid \textcumulative{} constraints.}

\( Q \gets \emptyset \)\tcc*{\( (C, k) \in Q \implies \) any subset of \( C \) of size \( k \) is a cover}
\( \mathcal{O} \gets \emptyset \)\tcc*{The set of all discovered \textcumulative{} constraints}
\For{\( C \in \mathcal{C} \)}{
    \tcc{Skip this cover if already mentioned in an earlier constraint}
    \If{\( \exists (C', k) : C' \supseteq C \land k \le \# C \)}{
        \Continue\;
    }
    \( \boldsymbol{\pi} \gets \boldsymbol{0}, \boldsymbol{\pi}(C) \gets 1, \pi_0 \gets \#C - 1 \)\tcc*{Invariant: \( (\boldsymbol{\pi}, \pi_0)\) is valid}
    \( I \gets \{ 1, \dotsc, n \} \setminus C \)\tcc*{Invariant: \( i \in I \implies \pi_i = 0 \)}
    \While{\( I \neq \emptyset \)}{
        \( i \gets \argmin_{i \in C} d_i \)\tcc*{Choose the shortest remaining task}
        \( \pi_i \gets \pi_0 - v^*(\boldsymbol{\pi}, \pi_0, I, A, \boldsymbol{b}) \)\tcc*{\( v^* \) is obtained from \Cref{eq:lifting-subproblem}}
        \( I \gets I \setminus \{ i \} \)\tcc*{Mark this variable as lifted}
    }
    \tcc{Mark all elements with unit coefficients as visited}
    \( Q \gets Q \cup (\{ i : \pi_i = 1 \}, \#C ) \)\;
    \tcc{Add the discovered \textcumulative{} unless it is dominated}
    \If{\( \exists r : \boldsymbol{a}_r \ge \boldsymbol{\pi} \land b_r  \le \pi_0 \)}{
        \Continue\;
    }
    \( \mathcal{O} \gets \mathcal{O} \cup \big\{ \cumulative{\mathcal{X}}{\boldsymbol{\pi}}{\boldsymbol{d}}{\pi_0}  \big\} \)\;
}
\tcp{Retain the \( N_\mathrm{out} \) best constraints with the highest capacity bound}
\( \mathcal{O} \gets \mathtt{TopN}(\mathcal{O}, N_\mathrm{out}) \)\;
\Return \( \mathcal{O} \)\;
\end{algorithm}

\section{Experimental evaluation}
\label{section:experiments}

In this section, I evaluate the performance of the presented approach on a collection of research benchmarks. I also compare the resulting bounds with those reported in the literature and show that the approach discovers new bounds for several benchmark instances (\nlowerbounds{} lower and \nupperbounds{} upper bounds), with \nsearchless{} lower bounds \emph{directly} corresponding to the capacity bounds of newly introduced constraints.

The implementation of the approach, together with the infrastructure for running the experiments and processing their results, is available in the supplementary material; the modeling part is based on the CPMpy library~\citep{Guns2019-Other}. All lifting subproblems (\Cref{eq:lifting-subproblem}) are solved with Gurobi 12~\citep{Other2024-Other-02}. I ran the experiments on \delftblue{}, with each run of an instance being allocated a single core of an Intel Xeon E5-6248R 24C 3.0GHz processor and 4000 MB of RAM with a time limit of one hour.

I use the benchmark collection previously used by \citet{Sidorov2025-Other-01} to evaluate the \uniteandlead{} approach for the same two problem models (RCPSP and RCPSP/max) with \textcumulative{} constraints and difference constraints of the form \( x_j - x_i \ge \gamma_{ij} \) for RCPSP/max and \( x_j - x_i \ge d_i \) for RCPSP, both having the latest completion time of all tasks (makespan) as the minimization objective. More specifically, I use 736 RCPSP instances from MiniZinc benchmarks~\citep{Stuckey2014-aiMag,Olaguibel1989-advancesInProjectScheduling,Baptiste2000-constraints,Carlier2003-eurJOperRes,Kone2011-computOperRes,Kolisch1997-eurJOperRes} and 349 RCPSP/max benchmarks distributed by PSPLIB~\citep{Kolisch1997-eurJOperRes} through C, D, UBO, and SM suites. All data files are available in the supplementary materials.

Given a minimization objective $\mathcal{O}$, each of the solvers is run with one of the two search directions. In the \textbf{primal search},\footnote{Also known as the linear SAT-UNSAT search.}  the solver generates a series of problems with an extra assumption \( \llbracket \mathcal{O} \le o_n \rrbracket\) for decreasing values of \( o_n \). In the \textbf{dual search},\footnote{Also known as the linear UNSAT-SAT search or destructive lower bound.} the solver generates a series of problems with an extra assumption \( \llbracket \mathcal{O} \le o_n \rrbracket\) for increasing values of \( o_n \). 

Each of the runs involves one of the following CP solvers: Pumpkin~\citep{Flippo2024-Other} and CP-SAT~\citep{Perron2024-Other}. The choice of the solvers is dictated by two factors: (a) both of the mentioned codes are advanced, state-of-the-art solvers performing well at the recent editions of MiniZinc Challenge, and (b) both of them support both the primal and the dual search directions. In the baseline runs, the solver is run on the input CSP as is, whereas the runs with the lifting first execute \Cref{alg:cover-enumeration} with \( N_\mathrm{cover} = 100 \), then \Cref{alg:cumulative-generation} with \( N_\mathrm{out} = 5 \), and proceed with the solver execution on the augmented CSP in the remaining time. I budget the preprocessing by limiting the number of candidate covers considered (\( N_\mathrm{cover} \)) and the number of added inferred constraints (\( N_\mathrm{out} \)), rather than by imposing a wall-time cutoff, since wall-time on shared HPC systems is not deterministic. All solvers have been invoked with the free search.

To evaluate the solver performance when it did not conclude optimality, I measure the rate of progress of a solver towards the best-known bound with the following metrics. First, if \( M(t) \) is the lowest makespan discovered at time \( t \in [0, T] \), and \( M^* \) is the lowest discovered makespan for this problem, then the \textbf{primal integral} is \( \int_0^T \frac{M(t) - M^*}{M(t)} \) and measures how fast the solver progresses towards good solutions~\citep{Berthold2013-operResLett}. Conversely, if \( B(t) \) is the highest lower bound discovered at time \( t \in [0, T] \), and \( B^* \) is the highest discovered lower bound for this problem, the \textbf{dual integral} is defined as \( \int_0^T \frac{B^* - B(t)}{B^*} \).

All source code necessary to reproduce the experiments, including the implementation of the lifting approach, as well as output data files, can be accessed in Zenodo \citep{Sidorov2026-Other-01, Sidorov2026-Other}.

\subsection{Evaluation against direct CP solver application}

First, I report the evaluation of the lifting approach for the instances where both the run performing lifting and its counterpart that does not lift (while having the same settings otherwise) have proven optimality. As shown on \Cref{fig:time-to-opt}, lifting is consistently able to cut down the solving time by large factors, sometimes from taking dozens of minutes to sub-second times. On the other hand, most of the instances with large (tenfold or more) slowdowns correspond to the cases where the lifting time dominates the solving time, which in practice can be avoided with an appropriate termination criterion for the lifting procedure.

\begin{figure}[ht]
\centering
\includesvg{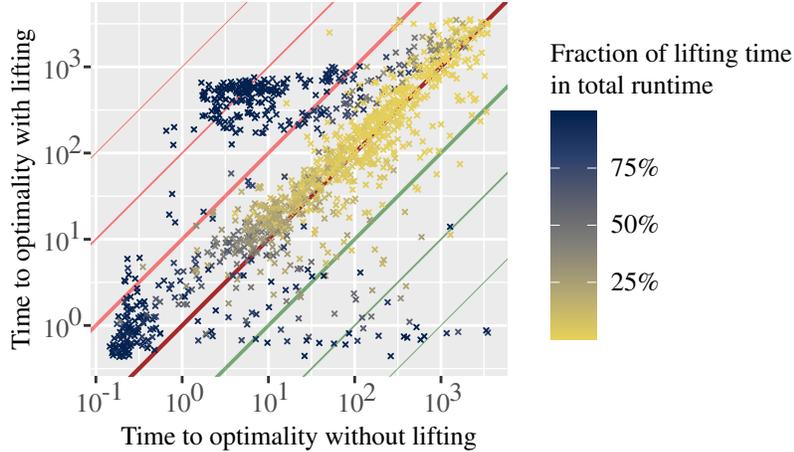}
\caption{Impact of using lifting on time to optimality. Every point corresponds to a pair of runs on the same instances, using the same underlying solver with the same search direction. Diagonal lines are evenly spaced and correspond to a tenfold relative change between the durations.}
\label{fig:time-to-opt}
\end{figure}

However, most of the instances were not solved to optimality with either approach, which means that making progress with further comparisons is more viable with primal and dual integrals as metrics. I proceed with examining this comparison on \Cref{fig:integrals-int:dual}, which supports the thesis that lifting performs as a helpful preprocessing routine: it substantially reduces the search effort on many instances, while not incurring a lot of \emph{search} overhead if it backfires. On the flip side (\Cref{fig:integrals-int:primal}), this is not true for primal integrals: in other words, adding new constraints does not help discover good solutions as consistently, but can hinder this process much worse.

\begin{figure}[ht]
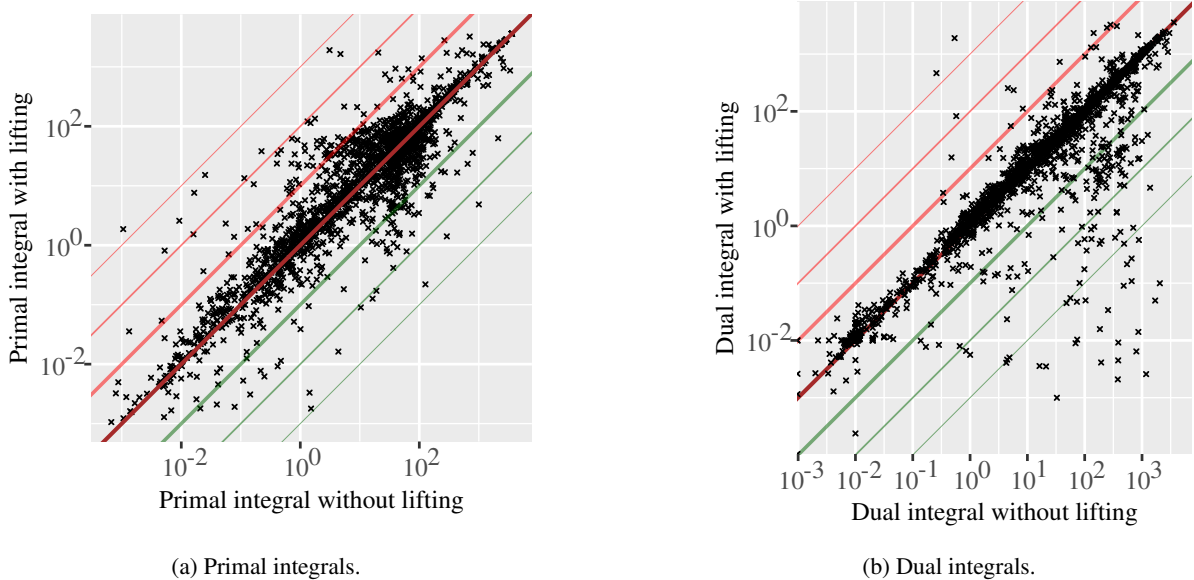

\centering
\subcaptionbox{Primal integrals.\protect\label{fig:integrals-int:primal}}[0.45\textwidth]{
    \resizebox{0.45\textwidth}{!}{
    \includesvg{experiments/primal-integrals-int}}
}
\hfill
\subcaptionbox{Dual integrals.\protect\label{fig:integrals-int:dual}}[0.45\textwidth]{
    \resizebox{0.45\textwidth}{!}{
    \includesvg{experiments/dual-integrals-int}}
}
\caption{Comparison of integrals between the baseline and lifting approach.}
\label{fig:integrals-int}
\end{figure}

These results already show an improvement over the \uniteandlead{}, since Sidorov et al. mention a much larger variance for dual integrals, especially on the degradation side. For example, they report that their approach degrades the search by at least an order of magnitude on many instances, whereas in the current evaluation, this is an \emph{exceptionally} rare event happening only in 19 pairs of runs. I investigate this in more detail in \Cref{table:degradation-quantiles}, which shows that even the “2× degradation” event is around the 95th percentile of observed degradations, whereas the opposite 5th percentile corresponds to 10× improvements for RCPSP benchmarks and to 50× improvements for RCPSP/max benchmarks.

\begin{table}[ht]
\caption{Percentiles of the observed distribution of the ratio of dual integrals between baseline runs and the corresponding runs with lifting across problem types.}
\label{table:degradation-quantiles}
\begin{tabular}{@{}cccccccc@{}}
\toprule
\textbf{Problem type} & \textbf{2.5\%} & \textbf{5\%} & \textbf{25\%} & \textbf{50\%} & \textbf{75\%} & \textbf{95\%} & \textbf{97.5\%} \\ \midrule
\textbf{RCPSP} & \( 2.1 \times 10^{-3} \) & 0.09 & 0.96 & 1.00 & 1.17 & 2.07 & 3.05 \\
\textbf{RCPSP/max} & \( 3.4 \times 10^{-3} \) & 0.02 & 0.71 & 1.03 & 1.30 & 1.98 & 2.42 \\ \bottomrule
\end{tabular}
\end{table}

I also observe that in many “high-performing” runs of lifting approach, the underlying constraints (or, at least, the one yielding the highest bound) commonly happened to be a \disjunctive{} constraint, that is, to have the right-hand side equal to one. In a similar vein, I investigate how the influence of lifting on the search depends on the best constraint, which is summarized in \Cref{fig:degradation-quantiles-rhs}. Indeed, the lifting runs that ended up on a \disjunctive{} constraint as the best one are responsible for the biggest relative gains; however, the runs that landed on a \textcumulative{} constraint still routinely produce large relative improvements, even if more modest. That said, discovering general \textcumulative{} constraints is only helpful when the capacity of the discovered constraint is not large; conversely, in the runs where the best constraint inferred through lifting had a capacity larger than eight units, the CP solvers have not been able to convert it into a significantly faster search.

\begin{figure}[ht]
\centering
\includesvg{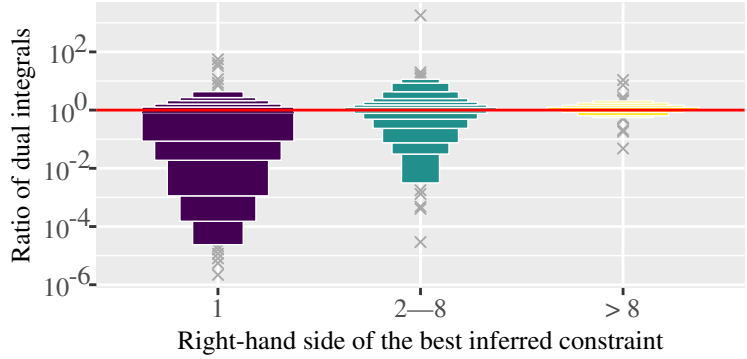}
\caption{Ratios of dual integrals between baseline runs and the corresponding runs with lifting across the right-hand side values of the best inferred constraint.}
\label{fig:degradation-quantiles-rhs}
\end{figure}

I present the runtime distribution of the preprocessing phase of the lifting approach (that is, \Cref{alg:cover-enumeration} and \Cref{alg:cumulative-generation} but not the CP solver) in \Cref{fig:preprocessing-time}. As can be seen on the histogram, under the budget imposed on lifting in this evaluation, all the lifting-specific work takes a few seconds in the majority of instances; on the other hand, the runs where lifting took more than a minute are covered exclusively by two collections with more than 500 tasks per instance, that is, UBO500 and UBO1000.

\begin{figure}[ht]
\centering
\includesvg{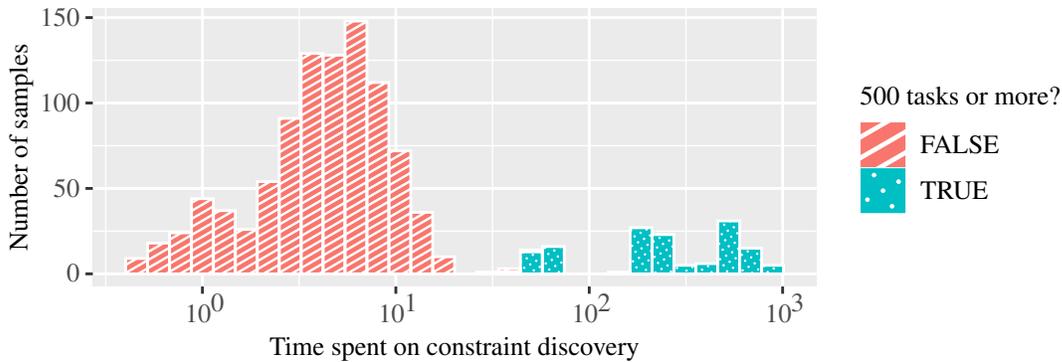}
\caption{Distribution of time required to run the inference of \textcumulative{} constraints.}
\label{fig:preprocessing-time}
\end{figure}

Last, I report the updated state-of-the-art bounds for these benchmarks discovered with the lifting approach in \Cref{appendix:bounds}, together with the criterion I used to choose which bounds to report. While many of the new bounds have been discovered during search, some are computed by taking the largest capacity bound among lifted constraints, thereby making them verifiable by inspecting the constraint (rather than retracing the reasoning made by a solver). While many such “search-less” bounds are encoded by \disjunctive{} constraints (and therefore could, in principle, have also been proven by \uniteandlead{} without any search), there is a bound for instance \#6 from the UBO1000 collection that improves upon the previously best known bound and is produced by a \textcumulative{} with the capacity of three units. I also note that no less than twelve of the instances in Pack and Pack-d collections~\citep{Carlier2003-eurJOperRes,Kone2011-computOperRes} can be closed immediately by using one of the lifted cumulative constraints, with the capacities varying between one and three.

\subsection{\texorpdfstring{Evaluation against \uniteandlead{}}{Evaluation against UɴL}}

I now proceed with the direct comparison with \uniteandlead{}; to this end, I run the following two solver configurations on each of the benchmark instances:
\begin{itemize}
    \item \uniteandlead{} with dual search in the same configuration as in Sidorov et al.
    \item Pumpkin with dual search, prefaced with lifting in the same configuration as in the previous subsection, \emph{but only considering covers with two elements}. This additional restriction is introduced to ensure that the lifting does not gain any advantage from discovering non-\disjunctive{} constraints.
\end{itemize}

As can be seen from \Cref{fig:unl-eval}, the proposed approach consistently outperforms \uniteandlead{} with respect to time to optimality for the applicable instances (\Cref{fig:unl-eval:time-to-opt}) and has varying performance and with respect to the dual integral (\Cref{fig:unl-eval:integrals}). This is explained by the fact that the lifting approach explicitly introduces \disjunctive{} constraints into the model, which are then used in specialized—and highly optimized—propagation algorithms. In contrast, \uniteandlead{} maintains all the facts around the disjointness of task pairs in a single graph and applies pruning for some of the cliques of that graph.

These plots also suggest that the lifting approach does not fully subsume \uniteandlead{}, as for some instances, lifting yields a notably worse performance. Interestingly, while many instances favorable for \uniteandlead{} are adequately explained by reasons unrelated to the methodology of this paper (such as the implementation of the preprocessing stage or discrepancies between the underlying solver versions), this is not true in all cases. For example, instance \#40 from the UBO50 collection is solved to optimality for \( \approx 30 \) seconds with \uniteandlead{} and for \( \approx 8 \) minutes with lifting, with \uniteandlead{} also doing 90× fewer constraint propagations until optimality. The reason for this is that the lifting approach chooses \( N_\mathrm{out} \) constraints with high capacity bounds (which in this case is the sum of durations) that also happen to share many variables, which translates into an insufficient extraction of the disjointness information; this reasoning is made more specific in \Cref{appendix:ubo50-40}.

\begin{figure}[ht]
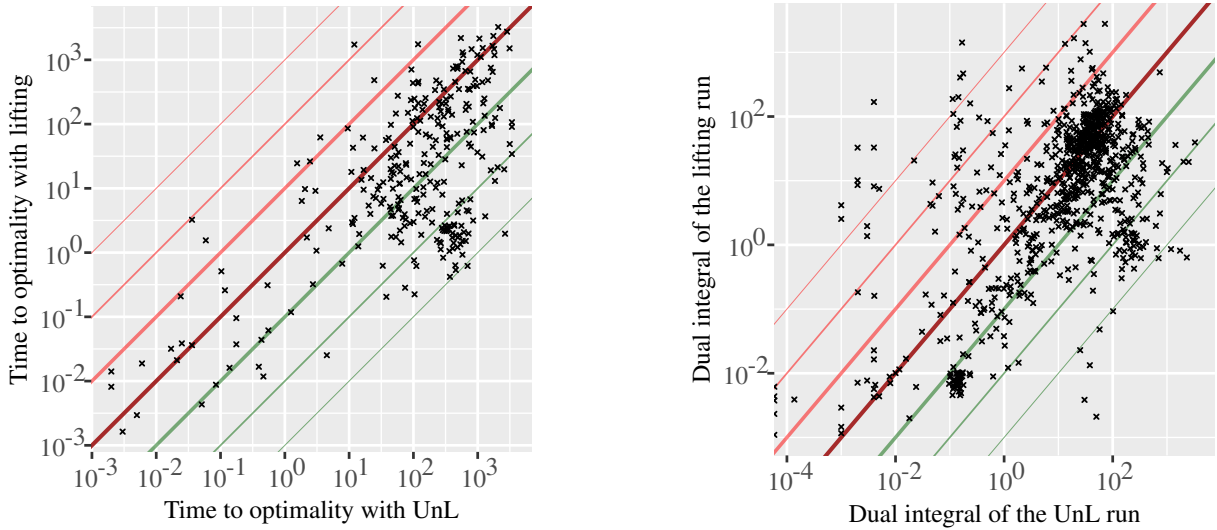

\centering
\subcaptionbox{Time to optimality for the instances solved to optimality with both approaches.\protect\label{fig:unl-eval:time-to-opt}}[0.45\textwidth]{
    \resizebox{0.45\textwidth}{!}{
    \includesvg{experiments/unl-time-to-opt}}
}
\hfill
\subcaptionbox{Dual integrals for all evaluated instances.\protect\label{fig:unl-eval:integrals}}[0.45\textwidth]{
    \resizebox{0.45\textwidth}{!}{
    \includesvg{experiments/unl-dual-integrals}}
}
\caption{Comparison between the lifting approach and \uniteandlead{}.}
\label{fig:unl-eval}
\end{figure}

\section{Related work}
\label{section:related-work}

The need to reason simultaneously about multiple resources has been recognized across optimization communities throughout the history of earlier developments. For instance, early exact approaches for RCPSP consistently employed global lower-bounding techniques that used the entire problem instance structure, such as the lower bounds mentioned by~\citet{Klein1999-eurJOperRes} based on solving relaxations to node packing (LB4) and parallel machine scheduling (LB5) problems. The presented approach can be seen as a generalization of these techniques in the following sense: \Cref{alg:cumulative-generation} with an appropriate choice of the input covers and an ordering of lifted variables can yield the same bounds as LB4 or LB5 implementations by letting a CP solver propagate the lifted \textcumulative{} constraints. However, unlike using LB4 or LB5 directly, the new \textcumulative{} remains known to a CP solver as one of the constraints rather than as an exogenous relaxation step, with any propagation algorithms for \textcumulative{} applicable to it. 

Within the CP community, the limitations of the “single-resource viewpoint” were also acknowledged for a long time. \citet{Beldiceanu2002-Other-01} introduced the \texttt{Cumulatives} constraint to capture multiple resources in a single constraint, which was later improved by faster filtering algorithms~\citep{Letort2013-Other, Letort2015-constraints}. While these works highlighted the need for multi-resource interaction, their focus was not on discovering new inference strategies but on running existing strategies faster, by replacing a “ping-pong” effect, where propagation for one \textcumulative{} triggers a propagation for another, with a singular pass accounting for all \textcumulative{} constraints. Conversely, the lifting approach can not only draw conclusions from several \textcumulative{} constraints at once but can also draw conclusions unreachable (without an exponential blowup) by reasoning over them in isolation.

Other efforts to mitigate the weaknesses of the single-resource viewpoint have focused on extracting stronger inferences from individual constraints by deriving non-dominated implied constraints. \citet{Carlier2007-eurJOperRes} investigated the systematic derivation of such constraints for a single \textcumulative{} by searching for one-to-one mappings from original resource consumptions to a reduced set of demands that preserve infeasibility while maximizing pruning. This logic was further refined by \citet{Baptiste2018-discreteApplMath} by examining the dual linear program of a preemptive relaxation of \textcumulative{} and introducing a mechanism for assigning a higher consumption level to a fixed number of tasks within a demand group. While these techniques strengthen the reasoning of a single resource, they remain fundamentally limited to intra-resource analysis. This work instead uses lifting as a tool for \emph{inter-resource} reasoning; rather than refining a single \textcumulative{} in isolation, the presented work derives global inequalities that capture the collective contention of tasks across the entire problem structure.

The earlier work by \citet{Sidorov2025-Other} implements a whole-problem reasoning scheme without merging constraints, opting instead to discover and communicate disjointness between tasks across resources. While \uniteandlead{} can theoretically emulate the node-packing bound (LB4) of Klein and Scholl, it lacks the expressive power to capture variable non-unit demands (LB5). Furthermore, it relies on managing auxiliary variables during search, which can be both a strength (by discovering a disjointness pattern conditional to a search state) and a hindrance (if there is no disjointness structure to be exploited). Conversely, the presented work compiles these interactions into conventional \textcumulative{} constraints before initiating any search by a CP solver.

The core mechanism that supports the contributions of this paper—lifting—originated in the integer programming community as a method for tightening cutting planes for knapsack constraints. First introduced by \citet{Padberg1975-operRes}, the technique iteratively strengthens an inequality by introducing variables not yet in the support. Later refinements, such as the ones proposed by \citet{Zemel1989-mathOperRes}, demonstrated that all sequential lifting coefficients could be computed efficiently via dynamic programming. Currently, lifting is a standard part of the modern MILP solver practice, as exemplified by the SCIP optimization suite~\citep{Hojny2025-arxivmathOc}, which extensively employs lifting procedures within its knapsack constraint handler to strengthen the linear relaxation of the problem.

In the context of scheduling, polyhedral methods have been originally applied to single-machine problems~\citep{Queyranne1993-mathProgram, Dyer1990-discreteApplMath}. For the general RCPSP, early polyhedral approaches~\citep{Olaguibel1993-eurJOperRes} focused on conflict variables (indicating a precedence within a pair of tasks) rather than start times, limiting their direct applicability in standard CP models. More recently, approaches combining CP propagators with LP relaxations have been explored, with CP-SAT~\citep{Perron2024-Other}, one of the state-of-the-art CP solvers, integrating scheduling cuts as special separation procedures for tightening LP relaxations. While scheduling cuts are constructed with a procedure similar to the one proposed here (choose a cover and lift it), those cuts are derived from \emph{energy}\footnote{The product of duration and resource consumption.} overflows; on the other hand, lifting as presented in this paper operates directly on the \emph{demands} via occupancy vectors, allowing us to view \textcumulative{} constraints as linear inequalities and apply standard lifting machinery directly to the start-time representation.

Collectively, these works support the thesis that the “single-resource viewpoint” is insufficient for complex scheduling benchmarks. While integer programming and specialized branch-and-bound approaches have long exploited whole-problem structures, CP has historically struggled to integrate these insights without sacrificing the modularity of propagators. By establishing a formal link between \textcumulative{} constraints and lifted linear inequalities, this work offers a pragmatic path to bring these powerful global inferences into CP solvers.

\section{Conclusions}
\label{section:conclusions}

This paper presents a novel approach for aggregating \textcumulative{} constraints by (i) reformulating them as linear constraints in the occupancy-vector representation, (ii) discovering sets of tasks that cannot be run jointly, and (iii) lifting the linear constraint blocking this group of tasks. The experimental evaluation shows that this approach is, at best, fundamentally helpful—or, at worst, not too obstructive—for proving objective bounds, although those advantages do not translate as well to discovering new high-quality solutions.

One possible direction for future work is a tighter coupling with the search loop, for instance, by adding \textcumulative{} constraints not \emph{before} the search but \emph{during} the search. The core technical problem here is that a \textcumulative{} lifted from a conflicting assignment does not necessarily propagate after backtracking, as observed in earlier work on linear constraint learning in CP~\citep{Nieuwenhuis2014-Other}. However, exploiting this procedure in a more heuristic way (e.g., during restarts) could be helpful to leverage the information learned during search.

Last, the running assumption in this work is that both durations and resource consumptions are constant. While this is valid for RCPSP and RCPSP/max, this is not true for several other scheduling problem formulations studied previously, most notably \emph{multi-mode} scheduling problems~\citep{Hartmann1998-networks}. However, extending this approach to multi-mode RCPSP—where task durations and demands are variable—will require handling the non-linearities resulting from the products of occupancy vectors and resource consumptions.

\section*{Acknowledgements}

I would like to thank Imko Marijnissen for the enlightening discussions that sparked the development of the technique presented in this paper.

\section*{Funding}

Konstantin Sidorov is supported by the TU Delft AI Labs program as part of the XAIT lab. 

\bibliographystyle{unsrtnat}
\bibliography{citations,citations-with-url}

\appendix

\section{Novel bounds}
\label{appendix:bounds}

I report the bounds discovered by the lifting approach if they are \emph{both} better than the previously reported bounds and are not directly reproducible without lifting. More precisely, I report bounds that are simultaneously (a) tighter than the bounds reported in the previous sources known to us that used the same benchmarks~\citep{Kolisch1997-eurJOperRes,Vilim2015-Other,Vilim2011-Other,StuckeyOther-Other,Sidorov2025-Other}, and (b) either tighter than any bound derived without preprocessing or matches it but was derived at least ten times faster than with any other approach.

Novel upper bounds (makespans) are reported in \Cref{table:novel-upper-bounds}, and novel lower bounds are reported in \Cref{table:novel-lower-bounds}; the same data is available as supplementary materials. All bounds are reported for RCPSP/max benchmarks; all reported durations are in \texttt{MM:SS} format. To indicate the remaining optimality gap, I also state the best known lower bound in \Cref{table:novel-upper-bounds} and the best known makespan in \Cref{table:novel-lower-bounds}; in either case, that bound is the tightest among the previously reported values and the values discovered without preprocessing. Additionally, \Cref{table:novel-root-lower-bounds} reports the lower bounds that can be verified by an appropriate lifted constraint.

\begin{table}[ht]
\centering
\caption{Novel upper bounds derived with lifting.}
\label{table:novel-upper-bounds}
\begin{tabular}{@{}cccccc@{}}
\toprule
\textbf{Collection} & \textbf{\#} & \textbf{Ref. objective} & \textbf{New objective} & \textbf{Time} & \textbf{Best bound} \\ \midrule
C & 63 & 366 & 363 & 43:50 & 347 \\
C & 67 & 350 & 349 & 23:56 & 346 \\ \midrule
UBO100 & 8 & 385 & 383 & 41:56 & 376 \\ 
UBO100 & 32 & 434 & 432 & 26:59 & 414 \\ \midrule
UBO200 & 4 & 893 & 838 & 29:25 & 605 \\
\bottomrule
\end{tabular}
\end{table}

\begin{table}[ht]
\centering
\caption{Novel lower bounds derived with lifting.}
\label{table:novel-lower-bounds}
\begin{tabular}{@{}cccccc@{}}
\toprule
\textbf{Collection} & \textbf{\#} & \textbf{Ref. bound} & \textbf{New bound} & \textbf{Time} & \textbf{Best objective} \\ \midrule
C & 63 & 343 & 348 & 43:08 & 363 \\
C & 66 & 340 & 347 & 49:16 & 368 \\
\midrule
UBO100 & 70 & 375 & 387 & 43:43 & 408 \\
\midrule
UBO200 & 2 & 682 & 731 & 46:25 & 813 \\
UBO200 & 3 & 482 & 542 & 30:47 & 906 \\
UBO200 & 5 & 499 & 559 & 57:28 & 767 \\
UBO200 & 6 & 538 & 569 & 38:29 & 765 \\
UBO200 & 8 & 505 & 583 & 48:53 & 911 \\
UBO200 & 32 & 753 & 794 & 51:45 & 863 \\
UBO200 & 33 & 784 & 785 & 56:14 & 834 \\
UBO200 & 34 & 595 & 673 & 35:16 & 774 \\
UBO200 & 65 & 728 & 747 & 35:22 & 814 \\
UBO200 & 70 & 724 & 823 & 57:43 & 877 \\
\midrule
UBO500 & 3 & 1159 & 1260 & 43:30 & 1808 \\
UBO500 & 6 & 1202 & 1350 & 01:05 & 2035 \\
UBO500 & 8 & 1175 & 1208 & 00:56 & 2212 \\
UBO500 & 38 & 1258 & 1398 & 01:12 & 2575 \\
\midrule
UBO1000 & 2 & 2321 & 2518 & 32:11 & 5479 \\
UBO1000 & 6 & 2311 & 2329 & 04:12 & 3508 \\
UBO1000 & 9 & 2274 & 2414 & 03:55 & 4231 \\
UBO1000 & 10 & 2367 & 2441 & 02:30 & 3702 \\
UBO1000 & 35 & 2382 & 2563 & 03:09 & 4737 \\
UBO1000 & 68 & 2478 & 2550 & 03:12 & 4688 \\
\bottomrule
\end{tabular}
\end{table}

\begin{table}[ht]
\caption{Novel \emph{search-less} lower bounds derived with lifting. The bounds marked with an asterisk have not been improved by a CP solver.}
\centering
\label{table:novel-root-lower-bounds}
\begin{tabular}{@{}ccccccc@{}}
\toprule
\textbf{Collection} & \textbf{\#} & \textbf{Ref. bound} & \textbf{New bound} & \textbf{Objective} & \textbf{Capacity} \\ \midrule
UBO200 & 3 & 482 & 531 & 906 & 1 \\
UBO200 & 4 & 514 & 583 & 838 & 1 \\
UBO200 & 5 & 499 & 533 & 767 & 1 \\
UBO200 & 6 & 538 & 558 & 765 & 1 \\
UBO200 & 8 & 505 & 559 & 911 & 1 \\ \midrule
UBO500 & 3 & 1159 & 1260* & 1808 & 1 \\
UBO500 & 4 & 1230 & 1259 & 2774 & 1 \\
UBO500 & 6 & 1202 & 1345 & 2035 & 1 \\
UBO500 & 8 & 1175 & 1328* & 2212 & 1 \\
UBO500 & 38 & 1258 & 1448* & 2575 & 1 \\ \midrule
UBO1000 & 2 & 2321 & 2515 & 5479 & 1 \\
UBO1000 & 6 & 2311 & 2315 & 3508 & 3 \\
UBO1000 & 9 & 2274 & 2418* & 4231 & 1 \\
UBO1000 & 10 & 2367 & 2482* & 3702 & 1 \\
UBO1000 & 35 & 2382 & 2568* & 4737 & 1 \\
UBO1000 & 68 & 2478 & 2536 & 4688 & 1 \\
\bottomrule
\end{tabular}
\end{table}

\section{Disjointness structure of instance \#40 from UBO50}
\label{appendix:ubo50-40}

As mentioned in the main text, \uniteandlead{} outperforms the lifting approach on this instance; this appendix provides a more specific explanation for this. \Cref{fig:ubo50-40} demonstrates the complement of the disjointness graph of that instance, with vertices corresponding to each of the fifty tasks that are scheduled in this instance, and edges connect pairs \( \{ u, v \} \) of vertices that can be executed in parallel (that is, \( a_{i,u} + a_{i,v} \le b_i \) for some \( i \)). Vertex colors are assigned as follows: \colorbox{red!30}{red} vertices correspond to tasks shared between all constraints added by lifting, \colorbox{blue!30}{blue} vertices correspond to tasks that are mentioned in some (but not all) constraints added by lifting, and \colorbox{green!30}{green} vertices correspond to tasks not mentioned in any lifted constraint.

The important insight in the structure of this graph is that it is (a) very sparse, with a random pair of tasks having 87\% probability of being disjoint, and (b) the green vertices are loosely constrained by the rest of the graph; to justify the latter, I observe that 235 out of 262 maximal independent sets of this graph involve a green vertex. This suggests that \uniteandlead{} is able to exploit all of the graph structure, which opens the inferences about the tasks corresponding to the green vertices that are not available to the CP solver after the lifting procedure concludes.

On the other hand, this shows the insufficiency of the selection step based solely on a quality metric of the \textcumulative{} constraint. In this example, the lifting procedure chooses many independent sets (equivalently, \disjunctive{} constraints) that contain all the red vertices and have a large total duration of tasks. Taken on a per-constraint basis, this is a good decision, as this set covers 40\% of the problem variables; unfortunately, none of the green vertices can be added to this set. I believe that resolving the issue demonstrated by this instance would require a more elaborate selection procedure that rewards the \emph{overall} coverage of the space of known conflicts, rather than the quality of each of the added constraints in isolation.

\begin{figure}[ht]
\centering
\includesvg[width = 0.66\textwidth]{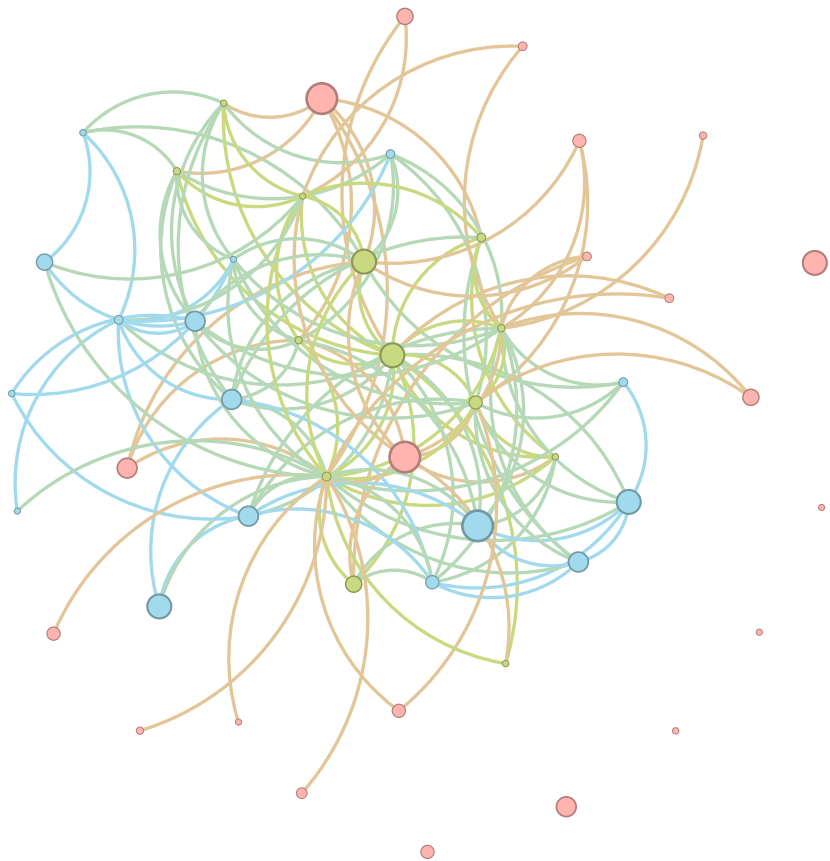}
\caption{The complement of the disjointness graph of instance \#40 from UBO50; vertex color encodes the number of lifted constraints which mention the corresponding task, and the area occupied by a vertex is proportional to the task duration.}
\label{fig:ubo50-40}
\end{figure}

\end{document}